\documentclass[11pt]{article}

\usepackage{fullpage} 
\usepackage{microtype}
\usepackage{graphicx}
\usepackage{subfigure}
\usepackage{booktabs} 
\usepackage{color}

\usepackage{natbib}
\usepackage{hyperref}

\usepackage{amsmath}
\usepackage{amssymb}  
\usepackage{mathtools}
\usepackage{amsthm}
 \usepackage[capitalize]{cleveref}
\usepackage{bm}

\newtheorem{assumption}{Assumption}
\newtheorem{definition}{Definition}
\newtheorem{theorem}{Theorem}
\newtheorem{corollary}{Corollary}
\newtheorem{lemma}[theorem]{Lemma}

\crefname{observation}{Observation}{Observations}

\DeclareMathOperator*{\E}{\mathbb{E}}

\newcommand{\set}[2]{\left\{#1\;\middle|\;#2\right\}}

\newcommand{\TV}[2]{\left\|#1-#2\right\|_{\mathrm{TV}}}
\newcommand{\mis}[3]{\mathrm{GCE}_{#1}(#3,#2)}
\newcommand{\mc}[3]{\mathrm{Mis}_{#1}(#3,#2)}
\newcommand{\coarsen}[2]{#1^{#2}}
\newcommand{\sparsity}{s}

\def\train{x_\mathrm{train}}
\def\trainvec{{\bf x}_{\mathrm{train}}}

\def\1{\mathbf{1}}
\def\supp{\mathrm{supp}}

\def\eps{\varepsilon}

\def\A{{\mathcal{A}}}
\def\B{{\mathcal{B}}}

\def\D{{\mathcal{D}}}

\def\P{{\mathcal{P}}}

\def\V{{\mathcal{V}}}

\def\reals{{\mathbb{R}}}
\def\nats{{\mathbb{N}}}

\newcommand{\prompt}[1]{\verb+#1+}

\usepackage{graphicx} 

\title{Calibrated Language Models Must Hallucinate}
\author{Adam Tauman Kalai\thanks{This work was done while the author was at Microsoft Research. Email:{\tt adam@kal.ai}}\\OpenAI \and Santosh S. Vempala\thanks{Supported in part by NSF award CCF-2134105 and a Simons Investigator award. Email: {\tt vempala@gatech.edu}}\\Georgia Tech}
\begin{document}

\maketitle

\begin{abstract}
Recent language models generate false but plausible-sounding text with surprising frequency. Such ``hallucinations'' are an obstacle to the usability of language-based AI systems and can harm people who rely upon their outputs. This work shows that there is an inherent statistical lower-bound on the rate that pretrained language models hallucinate certain types of facts, having nothing to do with the transformer LM architecture or data quality. For ``arbitrary'' facts whose veracity cannot be determined from the training data, we show that hallucinations must occur at a certain rate for language models that satisfy a statistical calibration condition appropriate for generative language models. Specifically, if the maximum probability of any fact is bounded, we show that the probability of generating a hallucination is close to the fraction of facts that occur exactly once in the training data (a ``Good-Turing'' estimate), even assuming ideal training data without errors. 

One conclusion is that models pretrained to be sufficiently good \textit{predictors} (i.e., calibrated) may require post-training to mitigate hallucinations on the type of arbitrary facts that tend to appear once in the training set. However, our analysis also suggests that there is no statistical reason that pretraining will lead to hallucination on facts that tend to appear more than once in the training data (like references to publications such as articles and books, whose hallucinations have been particularly notable and problematic)  or on systematic facts (like arithmetic calculations). Therefore, different architectures and learning algorithms may mitigate these latter types of hallucinations.
\end{abstract}

\section{Introduction}

The surprisingly high rate at which Language Models (LMs) generate false information, such as references to non-existent article titles, has recently emerged as a critical issue. The popular term \textit{hallucination} is defined in the \citet{merriam-webster_definition_2023} dictionary as ``a plausible but false or misleading response generated by an artificial intelligence algorithm.'' In one case, lawyers were fined \$5,000 for submitting legal research containing hallucinated legal cases that they believed were correct \citep{shin_humiliated_2023}. In healthcare, hallucinations could be life threatening to patients and physicians are concerned about malpractice cases \citep{mello_chatgpt_2023}. Furthermore, hallucinations have been widely reported on by the media \citep{weise_when_2023}, and the U.S.\ President recently put out an Executive Order calling for, among other things, safeguards against misleading outputs from generative AI systems \citep{biden_executive_2023}. 
This paper presents statistical lower-bounds on the rate of hallucination for LMs that are calibrated predictors of facts. This helps illuminate the nature of hallucination. It should \textit{not} be taken to mean that hallucination is inevitable. Rather, as we discuss it is consistent with the fact that practitioners have increasingly been augmenting ``pretraining'' procedures by ``post-training'' procedures that reduce hallucination rates at the cost of reducing calibration as well.

An LM is simply a probability distribution $D$ over sequences of \textit{tokens}, i.e., words or other character sequences. Clearly any LM which predicts every string with positive probability (a common property of LMs) will necessarily hallucinate with positive probability. However, if this probability is small, then hallucinations will be rare. Thus it is crucial to quantify the rate of hallucinations. Every distribution $D$ can equivalently be represented by its log-probabilities over entire sequences or conditional log-probabilities of the subsequent token given previous ones, 
$\log D(t_1 \ldots t_m) = \sum_{i=1}^m \log D(t_i \mid t_1 \ldots t_{i-1}).$ This mathematically trivial equivalence\footnote{The equivalence between next-token-prediction and document generation holds only if one disregards computational costs. It is similar to the statement that book can be output one word at a time, in order, in a single pass, even though writing a book generally involves many iterations. 
} has a profound implication: any LM can be used either to \textit{generate} text or \textit{predict} the next token in naturally occurring text conditioned on the previous tokens, though prediction and generation have different desiderata. 
For instance, consider the sentence,
\begin{quote}
    \textit{Alexa Wilkins had a tuna sandwich at Salumeria for lunch last Tuesday because the reviews said that it was divine.}
\end{quote}
Sentences of this sort could be likely under a predictive LM, for example, to offer suggestions to reduce typing on phones \cite[e.g.,][]{tanaka-ishii_word-based_2007}. It may be desirable to predict \textit{sandwich} as an option of a word to type after the word \textit{tuna}, along with other likely words such as \textit{salad}. On the other hand, the vast majority of sentences of this sort would be false if randomly fabricated by a generative LM. This paper shows LMs with good predictive text performance should hallucinate, even under ideal conditions. Notably in the first stage of pretraining, common today, the generative LM is optimized for predictive text performance \citep{chowdhery_palm_2022,openai_gpt-4_2023}.  Moreover, it gives a lower-bound on the rate of hallucination, which can shed light on the different rates at which different types of facts should be hallucinated.

What is common to both potential references and the example above (which we shall refer to as 5W = Who-Ate-What-When-Where-Why factoids), above is that they are \textit{arbitrary} in the sense that neither one can be determined systematically by rules---one cannot determine the veracity of most such facts that are not present in the training data. This in contrast to facts whose veracity can be determined systematically. We quantify how much LMs should hallucinate \textit{even in an simplified setting} with several ideal properties. Because we are giving statistical lower-bounds, we favor simplicity over generality as the point of our lower bounds is to identify a fundamental cause of LM hallucination. Similar to classification, where one seeks a lower-bound for the difficulty of classification in noiseless settings (but noise-tolerant classifications algorithms), we seek a hallucination lower-bound that holds in the simplest setting where training data is i.i.d.\ without factual errors. 

\paragraph{Calibration for generative models.} Calibration is a natural requirement of a probabilistic predictor meaning that its probabilities can be interpreted as accurate confidences in its own predictions. \citet{dawid_well-calibrated_1982} introduced the notion with the example of a weather forecaster: among days when they predict 30\% chance of rain, it rains approximately 30\% of the time. Calibration metrics have been extensively studied for LMs \citep[see, e.g.,][]{braverman_calibration_2020,jiang_how_2021,zhao_calibrate_2021}. \cref{fig:gpt4} illustrates multi-class calibration for GPT-4, a large modern LM, on a multiple choice exam. Post-training alignment was applied to reduce hallucination (among other factors) but was also found to reduce calibration \citep{openai_gpt-4_2023}. 
Calibration is both \textit{meaningful}, since a calibrated predictor's probabilities are interpretable as accurate confidences, and statistically \textit{achievable}.\footnote{Simple post-hoc probability modification procedures can calibrate any uncalibrated predictor and simultaneously increase its accuracy under metrics such as cross-entropy \citep[see, e.g.,][]{braverman_calibration_2020,blasiok_loss_2023}.}
In contrast, the perfectly accurate predictor would also be calibrated but may be impossible to learn.
However, calibration is only a minimal requirement for predictors, as not all calibrated models are useful predictors: the predictor which always outputs the annual average probability of rain is trivially calibrated. 

We provide a natural generalization of calibration to generative models. Our notion differs from prior uses of calibration in LMs which were at the token-level. The problem with analyzing raw token probabilities is that there are many ways to describe any fact in natural language, and thus having calibrated token probabilities is not particularly meaningful. To illustrate, consider the old \textit{trigram} LMs, which predict next-token probabilities based only on the previous two tokens (e.g., words). Trigram models are naturally calibrated at the token level, and yet hallucination was not a major problem for trigram models. This is because they mostly generate gibberish. 
Instead, our semantic-level calibration considers the probability distribution over pieces of information (facts or hallucinations) contained in the text.  We say an LM is calibrated if, for any probability $z \in [0,1]$, among the pieces of information it generates with probability $\approx z$, such information occurs on average in a $\approx z$ fraction of naturally occurring language (ideally the distribution from which training data was drawn). 

\paragraph{Why LMs hallucinate.} Hallucinations have mystified LM users and some researchers alike. \cref{sec:related}
surveys numerous hypotheses that have been studied for why LMs hallucinate, ranging from inaccurate or outdated training data to the next-token log-likelihood objective in training. Hallucination can also be due to an adversarial or out-of-distribution \textit{prompt}: a textual prefix provided for the LM to complete which establishes context. For example, there may be no factual  completion to a fabricated prompt such as, \textit{The 15 Elves of San Salami are named}\ldots.\footnote{A completion such as \textit{\ldots actually never mind, I have no idea} is unlikely given training data that does not have such retractions.} In contrast, our work shows that even in an ideal, unchanging world with perfect training data and no prompts, one should expect hallucinations from LMs which are calibrated.

\paragraph{Simplified setting.}
We consider a stationary language distribution $D_L \in \Delta(X)$ over documents  (i.e., strings of text) $x \in X$, and a learning algorithm $\A$ which takes training data  $\trainvec \in X^n$ consisting of $n$ documents sampled independently from $D_L$, and outputs an LM, i.e., a distribution $D_{LM}:=\A(\trainvec) \in \Delta(X)$. For simplicity, we assume that there are only facts in the training data, and at most one per document, i.e., no training hallucinations. We focus on \textit{arbitrary} facts such as the above examples, whose truth is usually undetermined from the training set, rather than \textit{systematic} facts such as $572 < 120523$ predictable from a training set by learning the underlying rules governing correctness. There is no statistical reason that LMs should hallucinate on systematic facts. Additionally, mistakes on systematic facts may not be considered hallucinations at all---they are often categorized as reasoning or arithmetic errors. 

We assume that each document $x \in X$ contains at most one \textit{factoid} $f(x) \in Y$, where factoids are arbitrary pieces of information which are each either true (facts) or false (hallucinations) and whose truth is statistically hard to determine from the training data. We also simplify matters by considering \textit{unconditional generation} \citep[e.g.,][]{tan_progressive_2021} in which the LM is sampled to generate text without any prompt (equivalently, the empty-string prefix). Again, compared to our simplified setting, hallucination may be even more likely in the more realistic case where the LM is prompted with contexts that come from a different distribution than the training data.


\paragraph{Results.} 
Consider $n$ i.i.d.\ samples from an unknown distribution $p$ over a large number of arbitrary factoids, such as the 5W example and references. The \textit{missing mass}, or in our case \textit{missing facts} $p(U)$, is the fraction of future samples from this fact distribution $p$ that were not observed in the $n$ training samples, where $U$ is the subset of facts that were unobserved in training data. The \textit{Good-Turing} estimate of the missing mass \citep{good_population_1953} is the fraction of samples (in our case facts) that appear exactly once in the training data. In our context, we call this the \textit{MonoFacts} estimator,
$$\widehat{MF} := \frac{\text{Number of facts appearing exactly once in training data}}{n}.$$
The Good-Turing estimator was shown to be within $|p(U) - \widehat{MF}| = \tilde{O}(\sqrt{1/n})$ of the missing mass, with high probability \citep{mcallester_convergence_2000, mcallester_concentration_2003}, for any distribution $p$. 

If the correctness of arbitrary factoids not contained in the training cannot be determined, the missing fact rate can provide a lower-bound on the rate of hallucination. This in turn gives us a lower-bound close to $\widehat{MF}$. In particular, under a ``regularity'' assumption on the factoid distribution, our simplest bound (\cref{cor:1}) implies that, for \textit{any} algorithm, with probability $\ge 99\%$ over training sets:
$$\text{Hallucination rate} \ge \widehat{MF} - \text{Miscalibration} - \frac{300|\text{Facts}|}{|\text{Possible hallucinations}|}-\frac{7}{\sqrt{n}},$$
where the ``Hallucination rate'' is the rate at which the LM generates hallucinations, the next term is ``monofact'' estimator of the  missing facts. The subsequent term is a ``miscalibration rate'' which quantifies how close to calibration the distribution is. We also have a term involving the ratio of the number of arbitrary facts to similar pieces of information that are false, which is exponentially small for many types of information. The final $6/\sqrt{n}$ term is small for the large training set sizes $n$ used in today's LM. 
The regularity assumption means that all unobserved factoids have, on average, equal probability of being true. More generally, the bound holds with probability $\ge 1-\delta$ where the constant 60 can be replaced by a term which is inversely proportional to $\delta$ and proportional to a regularity term on the distribution of factoids. The regularity term measures the ratio of the most likely factoid (that was not observed in training data) to the average unobserved factoid probability. This constant is 1 for symmetric distributions and other types of simple distributions. We relax it to consider bounded regularity, which permits  there to be certain negative correlations such as the fact that a person does not eat 1000 lunches in 1000 different places on the same day and allows for some factoids to have conditional probability 0, but it prohibits unobserved factoids from having very large probability.


\paragraph{Interpretation.}
Our lower-bounds have the following interpretation. First, one should identify a large set of factoids: arbitrary, plausible, regular factoids. These could be posts about 5W and plausible research article citations. Intuitively, there are exponentially more plausible factoids that are incorrect (which we call hallucinations), than those that are facts. One can then consider what fraction of such factoids might appear exactly once in the training data. In the case of 5W, one could imagine that half of the posts occur exactly once. This would suggest that a calibrated-factoid model would hallucinate on about half of its generations on 5W factoids. On the other hand, one could imagine that there are many many fewer than $n$ articles and since the goal of publication is advertising, each reference might be expected to appear multiple times in the training data\footnote{Even papers that are never cited often appear on an authors' CVs and institutional web pages and are cross-posted and listed in multiple publication indexes since the authors and institutions often wish to publicize their output and indices such as Google/Semantic Scholar may compete for completeness.} (i.e., have probability $\gg 1/n$) except for perhaps very recent publications (e.g., within the last day before they other references begin to appear). This would suggest that the missing mass for articles is low and that there is no inherent statistical necessity to hallucinate reference titles. There may be other causes of such hallucinations such as limited model capacity (even if the number of LM parameters is much greater than the number of articles, these parameters must encode many types of information beyond article titles). This also justifies the mitigation of consulting a fact database at generation time, even if that fact database is constructed solely from training data \citep{borgeaud_improving_2022}. \cref{cor:types} gives a simple generalization of our analysis to multiple \textit{types} of facts.

Conversely, one may wonder what if any facts appear only once in a large training corpus that might encompass the entire web. First, it is worth noting that significant efforts are made, in constructing LM corpora, to remove duplicates \citep[see, e.g.][]{shen_slimpajama-dc_2023}. And in other contexts, such as Zipfian and other power-law distributions, it is often observed that a constant fraction of entities appear exactly once, even as the datasets scale. While most facts that come to mind may be well-known, many public meeting notes and other posts may contain unremarkable facts that are not mentioned in other places. 

Despite this tension between factuality and predictive accuracy, the parameters of both types of LMs are typically trained or ``pretrained'' to maximize likelihood on a corpus, or equivalently to minimize the ``KL divergence'', a strong statistical discrepancy metric between the LM and the distribution of data on which it was trained.

\begin{figure}[t]
    \centering
    \makebox[0pt]{
    \includegraphics[width=0.49\linewidth]{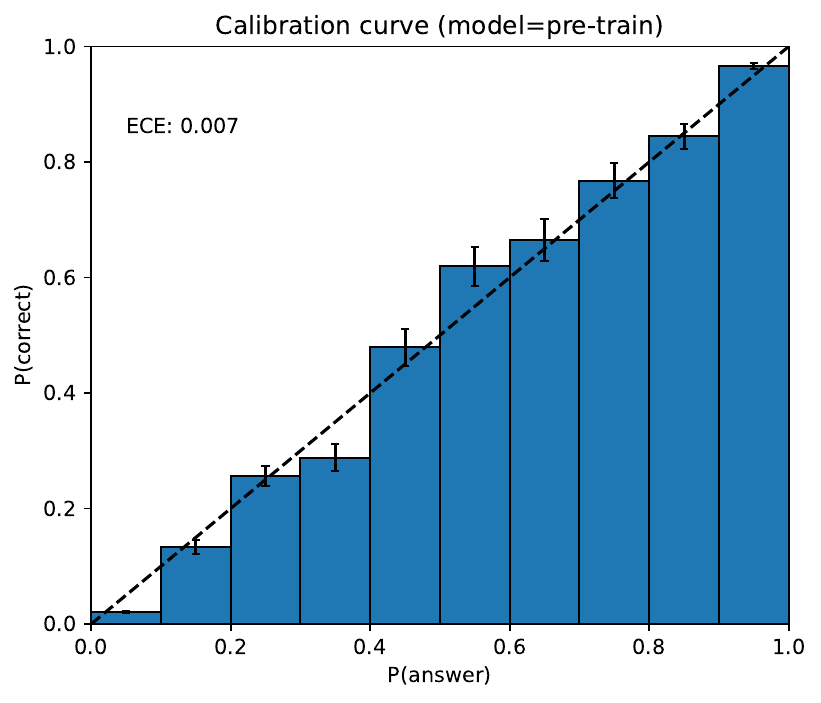}\hspace{0.02\linewidth}
    \includegraphics[width=0.49\linewidth]{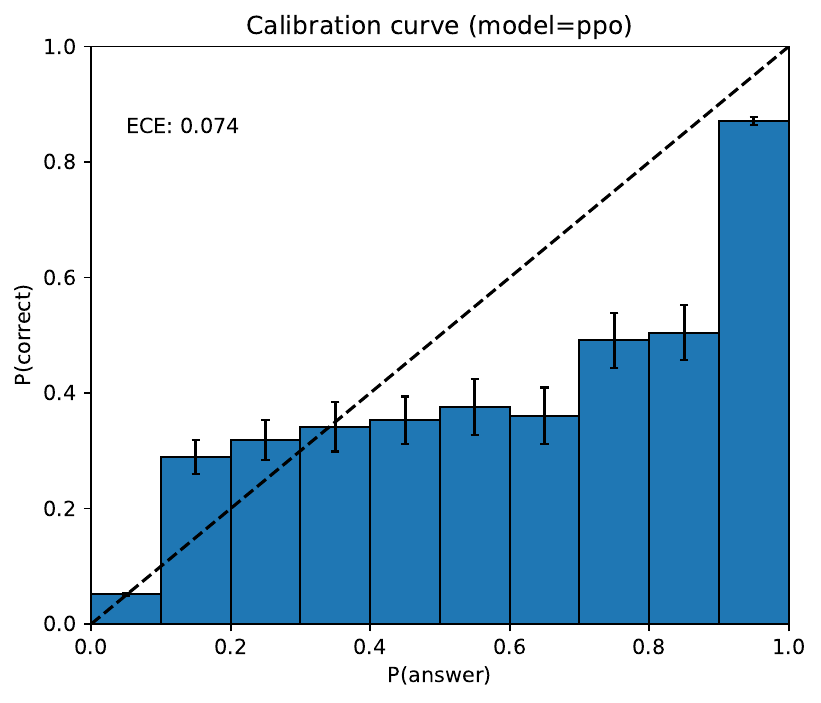}
    }
    \caption{GPT-4 calibration curves before (left) and after (right) reinforcement learning \citep[][Figure 8, reprinted with permission]{openai_gpt-4_2023}. As suggested by our model,
    post-training reduces hallucination rates at the cost of increasing calibration error. Note that calibration here is on a multiple-choice test rather than generative calibration which we study.}
    \label{fig:gpt4}
\end{figure}

\paragraph{Organization.}  Following related work discussed in \cref{sec:related}, \cref{sec:prelim} gives mathematical preliminaries including our definition of (mis)calibration. \cref{sec:model} defines our factoid model and \cref{sec:guarantees} states the main results. Sections \cref{sec:analysis} proves the main lemma. \cref{sec:upper} provides a simple upper bound showing that one cannot guarantee a hallucination rate greater than the monofact estimate for nearly calibrated estimators. 
\cref{sec:conclusions} concludes and discusses limitations and future work.  \cref{sec:good} in the appendix presents bounds on Good-Turing estimators from prior work. \cref{sec:alternatives} discusses extensions to alternative models such as calibration using bins of equal width, those based on raw accuracy (cross-entropy), or those using using prompts.  \cref{sec:corproofs} gives additional proofs.

\section{Related work}\label{sec:related}
As discussed in the introduction, the concept of calibration was introduced by \citet{dawid_well-calibrated_1982} and has been extensively studied in statistics and machine learning and even specifically for LMs \citep{braverman_calibration_2020,zhao_calibrate_2021,jiang_how_2021} and other related fields of deep learning and generative AI \citep{rohrbach_object_2018, maynez_faithfulness_2020}. \citet{blasiok_loss_2023} argue that calibration happens naturally as a byproduct of minimizing log-loss for deep neural networks, though their results are for a different architecture and different notion of calibration.

Unfortunately, there is not clear agreement on what counts as a hallucination. This is why we consider an idealized model in which there are clear-cut facts, where statements that violate these facts would generally be categorized as hallucinations by most definitions. 

\paragraph{Open- vs.~closed-domain hallucinations.}
Interestingly, many studies focus on hallucination with respect to a specific source document that is given to an LM, such as in translation \citep{xu_understanding_2023} or summarization \citep{maynez_faithfulness_2020}. There, LMs are also found to fabricate facts not present in the source document even when instructed to use only information present in the source. This is referred to as closed-domain hallucination in contrast to our open-domain setting, where the LM generates hallucinations which are not factually grounded in its training data. There is no clear statistical argument for why closed-domain hallucinations must occur, since if one can verify such hallucinations from the source text and generation one can avoid them by filtering out generations with such mistakes. Consistent with this, \citep{openai_gpt-4_2023} reports a greater reduction in closed-domain hallucinations over open-domain ones. 

\paragraph{Honesty vs.~factuality.}
\citet{evans_truthful_2021} points out that there is a difference between \textit{factuality} and \textit{truthfulness} (i.e., honesty). An LM which states something that disagrees with its training data may be said to \textit{lie}, whereas \textit{factuality} may be much harder to determine for a variety of reasons, including the possibility that what is considered factual may change as scientists make new discoveries and rebuke old theories. Nonetheless, in worlds where there is an absolute notion of facts, and the training data only contain facts, then any falsehood is also untruthful. Thus in ideal worlds where training data is perfectly consistent with an internally consistent ground-truth notion of facts, our bounds on non-factuality directly imply bounds on the rate at which LMs must generate untruthful information or ``lie.''

\paragraph{Hypotheses for why LMs hallucinate.}
There have been many explanations of why LMs hallucinate. The primary reason proposed for hallucinations is inadequacies in training data. This can be broken into two issues. First are falsehoods contained in the training data \citep{ji_survey_2023, lin_truthfulqa_2022, dziri_origin_2022, filippova_controlled_2020}, which \citet{lin_truthfulqa_2022} refer to as ``imitative falsehoods.'' Second is the temporal nature of data, i.e., the fact that training data is no longer relevant and is missing current events \citep{aksitov_characterizing_2023,vu_freshllms_2023}. While both of these are certainly factors in hallucination, our work shows that it is not the only cause of LM hallucinations. 

An additional reason given for why LMs may hallucinate is the fact that they are trained to produce tokens one at a time may lead to hallucination \citep{zhang_how_2023} because the LM may generate a plausible-sounding sequence of tokens which is impossible to complete factually. While this may be true for transformers due to computational limitations, this is not simply a statistical byproduct of their being trained on the next-token prediction task. Since document log-likelihood is simply the sum of next-token log-likelihood, the two objectives are identical and thus from a statistical perspective this is not a factor. Mathematically, any probability distribution over documents can be represented as a conditional next-token probability distribution.

Another line of work that sheds light on the nature of hallucinations shows that LMs \textit{know when they're hallucinating} \citep{kadavath_language_2022,azaria_internal_2023,agrawal_language_2023, manakul_selfcheckgpt_2023}. Various techniques may be used to identify LM hallucinations purely from the LM itself, either the internal weights, its token probabilities, or by querying it with additional questions. This is consistent with our work in the sense that one would expect even an ideal ``super-inelligent'' model should hallucinate if its goal is predictive accuracy.

A bevy of additional reasons have been proposed and studied for why LMs may hallucinate. These fall under headings such as duplicates in training data, biases, architecture, overconfidence and various types of knowledge failures, among others. A complete listing of these is beyond this scope of this work. For further details, see the recent surveys of
\citet{huang_survey_2023,zhang_sirens_2023,ji_survey_2023}.

\section{Mathematical Preliminaries}\label{sec:prelim}

We first define preliminary notation and concepts, including statistical distance and our notion of generative calibration.

\subsection{Standard definitions and notation}
Let $\Delta(S)$ denote the set of probability distributions over $S$.  For distribution $\D \in \Delta(S)$ and $R \subseteq S$, let $\D(R) := \sum_{x \in R} \D(x)$. 
The \textit{Total Variation distance} (TV) (also called statistical distance) between distributions $\D, \D' \in \Delta(S)$ has multiple equivalent definitions (for finite or countably infinite sets $S$):
\begin{equation}\label{eq:tv}
\|\D-\D'\|_{\text{TV}} := \max_{R \subseteq S} |\D(R)-\D'(R)|=\frac{1}{2}\sum_{x \in S} |\D(x)-\D'(x)|=\sum_{x \in S} \left(\D(x)-\D'(x)\right)_+,\end{equation}
where $(z)_+=\max(0, z)$ for $z \in \reals$. It satisfies $\|\D-\D'\|_{\text{TV}}=\|\D'-\D\|_{\text{TV}} \in [0,1]$. The \textit{support} of distribution $\D$ is $\supp(D)=\{x \in S\mid \D(x)>0\},$ as long as $S$ is finite or countably infinite. Let the set of \textit{partitions} of set $S$ be denoted by $\P(S) := \{ \Pi\subseteq 2^S \mid \forall x \in S ~|\{B \in \Pi \mid x \in B\}|=1\}$. For a function $f:X \rightarrow Y$ and set $S \subseteq X$, let $f(S) := \{f(x) \mid x \in S\}$. For a distribution $\D \in \Delta(X)$, let $f \circ \D$ denote the induced distribution of $f(x)$ for $x \sim \D$, i.e., $f(y) := \sum_{x: f(x)=y} \D(x)$. Finally, let $\D^{\times n}$ denote the distribution over $n$ independent samples from $\D$.

\subsection{Generative (mis)calibration}\label{sec:calibration}

This section defines a notion of \textit{miscalibration} $\mc{b}{g}{p} \in [0,1]$ for a generative model $g$ that measures how accurate its own probabilities are with respect to future examples drawn from a given pretraining distribution $p$, using $b \ge 1$ bins. It is a natural extension of existing notions of calibration to generative models, and can be skipped on first read for those who want to get straight to the model. \cref{sec:alternatives} discusses the relationship between this and existing notions of calibrated classifiers. The definitions in this section apply to any distributions $p, g \in \Delta(Y)$ for any finite set $Y$. In other words, there is only assumed to be a generative distribution $g$ over information $y \in Y$, and calibration is with respect to a ``true'' distribution $p$. Finiteness of $Y$ is only assumed at this point to avoid measure-theoretic technicalities. We first define a calibrated distribution $g$ as any \textit{coarsening} of $p$.

\begin{definition}[Coarsening and calibration]\label{def:coarsening}
    For partition $\Pi \in \P(Y)$ and $\D \in \Delta(Y)$, $\D^\Pi \in \Delta(Y)$ is the $\Pi$-\emph{coarsening of} $\D$ if,
    $$\forall B \in \Pi~\forall y \in B~~~~\D^\Pi(y)=\frac{\D(B)}{|B|}.$$    
    Distribution $g$ is said to be \emph{calibrated} to $p$ if $g=\coarsen{p}{\Pi}$ for some partition $\Pi$.
\end{definition}

Clearly $g=p$ is calibrated to $p$, and so is the uniform distribution $g(y)=u(y):=1/|Y|$. To define \textit{miscalibration}, 
let $B^g_z := \{y \in Y \mid g(y) = z\}$ and omit $g$, writing $B_z=B^g_z$ when clear from context. It is also clear that $g$ is calibrated to $p$ iff $g=\coarsen{p}{\B(g)}$ for partition $\B(g):=\{B_z \mid z \in [0,1]\}$ and thus a natural definition of miscalibration (which is 0 iff $g$ is calibrated to $p$) is: 
\begin{equation}\label{eq:mcinfty}
\mc{\infty}{p}{g}:=\bigl\|\coarsen{p}{\B(g)}-g\bigr\|_{\text{TV}} = \frac{1}{2}\sum_{B \in \B(g)}\sum_{y \in B}\left|\frac{p(B)}{B}-g(y)\right|.\end{equation}
The $\infty$ in the above notation refers to the fact that there is no limit on the number of bins. 
This also explains why it is called ``calibration'': the average probability over each bin that shares a common value of $g(y)$ is $g(y)$. This can be written as,
\begin{equation}\label{eq:cal}
\forall z\in [0,1]~~\E_{y\sim g}\bigl[p(y) \mid g(y)=z\bigr]=z.
\end{equation}

We next generalize $\B(g)$ to intervals. For $I \subseteq [0,1]$, define: 
$$B_I=B^g_I := \{y \in Y \mid g(y) \in I\}.$$
The definition below uses $b \ge 1$ adaptively sized bins which each have roughly equal probability mass in terms of generations $y \sim g$. \cref{sec:calvar} gives an alternate definition using bins of equal width in terms of log-probability.
\begin{definition}[Miscalibration]\label{def:mc}
Let $b \in \nats$ and define the adaptive partition,
$$\V_b(g):=\bigl\{B_{[0, t_1]}, B_{(t_1, t_2]}, \ldots, B_{(t_{b-1}, t_b]}\bigr\} \text{ where }t_i=\sup\set{z \in [0,1]}{\sum_{y:g(y) \le z} g(y) \le i/b}.$$
The \emph{miscalibration} of $g$ with respect to $p$ is $\mc{b}{p}{g}:=\TV{\coarsen{p}{\V_b(g)}}{g}$.
\end{definition}
Adaptive binning has been used in supervised classification \citep[e.g.,][]{nixon_measuring_2020}. Note that $\mc{1}{p}{g}=\TV{g}{u}$ is the total variation to the uniform distribution, which shows that $\mc{b}{p}{g}$ is not monotonic in $b$ because $b=1$ is the minimum for $g=u$ (regardless of $p$) while $b=\infty$ minimizes $\mc{b}{p}{p}=0$ for $g=p$. Also, $\mc{b}{p}{g}=\mc{\infty}{p}{g}$ for $b$ is large enough that $1/b\le \min_{y\in\supp(g)} g(y)$. Finally note that necessarily $t_b=1$ in the above definition.

\paragraph{Advantages and limitations.}
An advantage of semantic-level calibration is that it naturally models the exponentially many ways there are to describe any given fact, unlike token-level calibration models. This is also a disadvantage, however, because it means that it may be intractable to measure calibration rates, and thus experimental validation may be easier in synthetic settings where facts have canonical descriptions. One nice property of the above adaptive binning is that $\mc{b}{p}{g}=0$ iff $g$ is calibrated to $p$, regardless of $b$, whereas other definitions give 0 miscalibration whenever there is a single bin. Nonetheless, \cref{sec:calvar} shows how our analysis applies to other binning strategies.


\section{The model and guarantees}\label{sec:model}
Our notation is summarized in \cref{tab:common}. Let $\Sigma$ denote the set of tokens and $\Sigma^*$ denote the set of finite token sequences, i.e., strings. Let $X \subseteq \Sigma^*$ denote the set of documents (this could be finite, e.g., if a length restriction is imposed or countably infinite if $X=\Sigma^*$). 

The world is assumed to contain randomness, which is modeled by a distribution $D_\text{world} \in \Delta(\Delta(X))$ over language (document) distributions $D_L \sim D_\text{world}$. (More generally, a full world model would contain lot of other information but language distributions suffices for our purposes.) The training data consists of $n$ documents $\trainvec \sim D_L^{\times n}$ sampled independently from this distribution. It will be convenient to denote by $D_\text{train} \in \Delta(X^n)$ the distribution over training data, where the probability of any training document is:
$$D_\text{train}(\trainvec) := \E_{D_L \sim D_\text{world}}\left[\prod_{i=1}^n D_L\bigl(\train^{(i)}\bigr)\right].$$
This model requires a static world that does not change. Of course, real-world distributions are non-stationary and dynamic, and real-world training data contains duplicates and is not i.i.d. However, our lower bounds imply hallucination even in this ideal static setting.

\subsection{Factoid assumptions}
We assume there is a large but finite set $Y$ of ``factoids'' by which we mean \textit{arbitrary plausible} pieces of information, each of which could be true (facts) or false (hallucinations), as determined by some world distribution. Their arbitrary nature means that, given the facts observed in training data, one cannot infer any likely facts among the unobserved factoids. Of course, many real-world facts are systematic, not arbitrary. For instance, mathematical inequalities such as $17 < 252395$ are systematic and should thus not be included in $Y$. Note that $Y$ is not intended to capture all world facts but rather a large set of arbitrary ones. It could contain the 5W factoids (except for people who eat the same lunch every day in the same location, as their eating behaviors are too systematic). There is no statistical reason an LM must hallucinate on systematic facts.

We will make a few assumptions about factoids. 
\begin{enumerate}
\item {\bf One-per-doc}: First, we assume that there is at most one factoid per document by defining a surjective function $f: X \twoheadrightarrow Y$ which extracts a single 
factoid with each document, where $f(x)=\bot$ represents the empty fact (to allow for documents with no facts) and assume $\bot \in Y$. This makes the notation simpler as one can define the induced factoid distribution $p=f \circ D_L \in \Delta(Y)$ defined by $p(y) := \sum_{x: f(x)=y} D_L(x)$. Similarly, we can take $g = f \circ D_{LM}$ to be the induced distribution over generated factoids, where $D_{LM} \in \Delta(X)$ is the distribution over documents generated by the LM. The surjective requirement simply means that every factoid is describable by some document $x \in X$, and if this didn't hold one could always shrink $Y$. Our model does permit many documents to encode the same factoid, since there are typically many ways to describe any given piece of information. Again, this assumption may be generalized to a model where documents contain sets of factoids with further notation, but we show that calibrated LMs hallucinate even in a simple idealized setting with one factoid per document.

\item {\bf Good-training-data}: Second, we assume that the set of facts $F:=\supp(p) \cup\{\bot\}$ where $\supp(p)=\{y \in Y \mid p(y) > 0\}=f(\supp(D_L))$, i.e., the training data consists solely of facts and every fact has non-zero probability of being described under $D_L$. Both of these assumptions can be removed at the cost of additional notation without changing the results in the slightest---the world distribution would need to determine an arbitrary $F \subseteq Y$ and $D_L \in \Delta(Y)$. Keeping in the spirit of the ideal training data model, we choose to simplify notation and take $F:=\supp(p) \cup \{\bot\}$. The set of hallucinations is $H:=Y \setminus F$, i.e., every non-fact is defined to be a hallucination. 

\item {\bf More-false-than-true}: Third, \cref{ass:valid} below requires that there are many more falsehoods in $Y$ than facts, which makes sense for many types of information. For example, consider those factoids which describe a paper citation including the paper title, authors, and further details. There are vastly more plausible citations than existing ones. In this case, one may choose to include in $Y$ a somewhat smaller sparse set, i.e., not include author middle names or other minor variations in $Y$ which would make a reference in the ``grey area'' between fact and hallucination. Similarly, for 5W factoids, there are many more combinations of people who did not eat a given food on at a given time than people who did. 

\begin{assumption}[$\sparsity$-Sparse facts]\label{ass:valid}
There are many fewer facts than hallucinations. Specifically, there exists $\sparsity \in \reals$ such that, with probability 1 over $D_L\sim D_\mathrm{world}$: 
    $$|F| \le e^{-\sparsity}|H|.$$
Recall that $F:=f(\supp(D_L)) \cup \{\bot\}$ and $H=Y \setminus F$.
\end{assumption}
We write sparsity as an exponential to reflect the general exponentially nature of natural language facts.

\item {\bf (Semi)Regularity}: Finally, and most importantly, we assume that no single factoid is very likely. Specifically, perfect regularity requires that after observing the training data, all unobserved factoids are equally likely to appear as facts in the language distribution, and we also provide a relaxed $r$-regular notion, which we refer to as a ``semi-regularity'' assumption.
\begin{definition}[Regular facts]\label{def:balfacts}
$D_\mathrm{world}$ has \emph{regular facts} if for all $\trainvec \in \supp(D_\mathrm{train})$:
$$
\forall y, y' \in U~~~\Pr[y \in F \mid \trainvec]=\Pr[y' \in F\mid \trainvec].$$
For $r \ge 1$, $D_\mathrm{world}$ has $r$-\emph{regular-facts} if for all $\trainvec \in \supp(D_\mathrm{train})$,
$$\forall y \in U~~~\Pr[y \in F \mid \trainvec] \le  \frac{r}{|U|}\E\bigl[|F \cap U| \; \bigm| \;\trainvec\bigr].$$
\end{definition}
Having regular facts will suffice to prove the lower bound \cref{cor:balfact}, but stronger lower bounds will be possible if we also have regular probabilities.
\begin{definition}[Regular probabilities]\label{def:balprobs}
$D_\mathrm{world}$ has \emph{regular 
 probabilities} if for all $\trainvec \in \supp(D_\mathrm{train})$:
$$
\forall y, y' \in U~~~\E[p(y) \mid \trainvec]=\E[p(y') \mid \trainvec].
$$
For $r \ge 1$, $D_\mathrm{world}$ has $r$-\emph{regular-probabilities} if for all $\trainvec \in \supp(D_\mathrm{train})$,
$$\forall y \in U~~~\E[p(y) \mid \trainvec] \le  \frac{r}{|U|}\E[p(U) \mid \trainvec].$$
\end{definition}
Finally, we say $D_\mathrm{world}$ is \emph{regular} if it has regular facts and regular probabilities. It is not difficult to see that a distribution is regular iff it has 1-regular-facts and 1-regular-probabilities. We also note that our regularity assumptions could be modified to only holds with high-probability over $\trainvec$ and not for all $\trainvec \in \supp(D_\mathrm{train})$. 
\end{enumerate}

\bigskip

\begin{table}[t]
    \centering
    \begin{tabular}{rl}
        Symbol & Meaning\\\toprule
        $X \subseteq \Sigma^*$ & The set of documents (strings)\\
        $Y$ & Set of \textit{factoids}, arbitrary plausible pieces of information\\
        $\bot \in Y$ & Special ''empty string'' fact representing\\
        $f: X \twoheadrightarrow Y$ & Each document $x$ contains a single factoid $f(x)$, and $f(X)=Y$\\
        $D_\text{world} \in \Delta(\Delta(X))$ & Distribution over document distributions\\
        $D_L \in \Delta(X)$ & Language distribution over documents $D_L \sim D_\text{world}$\\
        $\trainvec\in X^n$ & Training data (i.i.d.) $\trainvec=\bigl\langle \train^{(1)},\train^{(2)},\ldots, \train^{(n)}\bigr\rangle\sim \D_L^{\times n}$\\
        $\A: X^n \rightarrow \Delta(X)$ & Algorithm that learns an LM from training data\\\midrule
        $p\in \Delta(Y)$ & Distribution over factoids $f(x)$ for documents $x\sim D_L$, i.e., $p:=f \circ D_L$\\
        $F \subseteq Y$ & Facts $F:=\supp(p) \cup \{\bot\}$, non-hallucinations\\
        $H\subseteq Y$ & Hallucinations $H:=Y \setminus F$\\        
        $D_\text{train} \in \Delta(X^n)$ & Distribution over $\trainvec \sim \D_L^{\times n}$ induced by $D_L\sim D_\text{world}$.\\
        $O \subseteq Y$ & Set of observed factoids $O:=\bigl\{f\bigl(\train^{(i)}\bigr) \mid i=1,2,\ldots, n \bigr\} \cup \{\bot\}\subseteq F$\\
        $U \subseteq Y$ & Unobserved factoids $U := Y \setminus O \supseteq H$\\        
        $\nu$ & Posterior over $p$ given training data $\trainvec$\\         
        $D_{LM} \in \Delta(X)$ & Distribution over documents generated by the LM, $D_{LM}:=\A(\trainvec)$\\
        $g\in \Delta(Y)$ & Distribution over factoids $f(x)$ for documents $x \sim D_{LM}$, i.e., $g:=f \circ D_{LM}$\\
        \bottomrule
    \end{tabular}
    \caption{Summary of notation. Symbols below the line are all derived quantities in terms of symbols above the line.  \label{tab:common}}
\end{table}

We illustrate a simple family of regular world distributions which satisfy our assumptions, followed by one which is only $r$-regular, does not have independence, and has anti-correlations between facts. 

\paragraph{Regular example: Permuted power-law world.} Suppose $X=Y$ and $f$ is the identity.  The world distribution $D_\text{world}$ first picks $F \subseteq Y$ uniformly random over such that $|F| = N$ (where $N \le |Y|/(1+e^\sparsity)$ so that $|F|/|H| \le e^{-\sparsity}$ so facts are $\sparsity$-sparse) and then picks $p$ to be a power-law distribution supported on a random ordering on $F$. That is, it picks a random permutation $\sigma : \{1,2,\ldots, N\} \hookrightarrow F$ and defines $p$ so that $p(\sigma(i)) \propto i^{-k}$ for some constant $k \ge 0$. This includes the uniform distribution over $F$ ($k=0$) and Zipfian distributions ($k=1$) as special cases. 

\paragraph{Semi-regular example: W5 with negative correlations.} 
The set of factoids is the product of fixed sets of people, dates, foods, and locations. $D_\text{world}$ chooses the set of facts $F$ randomly by: for each person on each date, there is a single fact which consists of that person eating a random food at a random location, and $p$ is uniform over $F$. This creates \textit{anti-correlations} between factoids because the knowledge about what and where a person ate a specific meal rules out any possible alternatives for that meal. Nonetheless, it can be seen that $r$-regularity holds for $r \le n_\text{people} n_\text{dates}$.

The above model can be enriched in various ways by adding reasons and by modeling people's preferences, geographic constraints, and behaviors. However the regularity parameter $r$ will be prohibitively large if there are predictable eaters, and indeed LMs might hallucinate less if there are large numbers of predictable eaters because an LM learning algorithm may learn these patterns and hallucinate less often.

\section{Guarantees}\label{sec:guarantees}

Our results are stated in terms of missing facts, a term inspired by the missing mass in data drawn from a probability distribution \citep{good_population_1953}. The \textit{missing facts rate} is the fraction of facts (according to the pretraining fact distribution $p$) that were not observed in the training data. Specifically, it is defined to be $p(U)$ where $U$ is the set of unobserved facts in the training data. Formally, define the \textit{observed} and \textit{unobserved} factoids as,
\begin{equation}\label{eq:ou}
O=O_{\trainvec} := \set{f\bigl(x^{(i)}_\mathrm{train}\bigr)} {i=1,2,\ldots, n}\cup \{\bot\}\subseteq Y, ~~~U=U_{\trainvec}:=Y \setminus O_{\trainvec},
\end{equation}
respectively. The \textit{monofact estimate} of the missing fact rate is defined to be the fraction of facts that appear exactly once in the training data:
\begin{equation}\label{eq:mf}
    \widehat{MF} = \widehat{MF}_{\trainvec} := \frac{\bigl|\bigl\{y \in Y \setminus \{\bot\} \mid y =f\bigl(\train^{(i)}\bigr) \text{ for exactly one }i \in \{1,2,\ldots n\}\bigr\}\bigr|}{n}.
\end{equation}
Note that the facts in the training data are distributed according to the distribution $p$. \cref{sec:good} states classical results asserting that $|\widehat{MF}-p(U)|=\tilde{O}(\sqrt{1/n})$ with high probability over $n$ samples from any distribution $p$.

An \textit{algorithm} $\A: X^n \rightarrow \Delta(X)$ takes as input $n$ training documents and outputs a document distribution $D_{LM}=A(\trainvec)$, which determines $g=f \circ D_{LM}$, i.e., $f(x)$ for $x \sim D_{LM}$. We now state our first result, which relies on a regular $D_\text{world}$ defined in \cref{def:balfacts,def:balprobs} above.
\begin{corollary}\label{cor:1}
Fix any $\delta\in [0,1], b, n\in \nats, \sparsity \in \reals$ and any $\sparsity$-sparse regular $D_\mathrm{world}$. Then for \textit{any} algorithm $\A: X^n \rightarrow \Delta(X)$, with probability $\ge 1-\delta$ over $D_L\sim D_\mathrm{world}$ and $\trainvec \sim D_L^{\times n}$,
$$g(H) \ge \widehat{MF} -\mc{b}{p}{g} -\frac{3e^{-\sparsity}}{\delta}-\sqrt{\frac{6\ln(6/\delta)}{n}},$$
where $D_{LM}=A(\trainvec)$, $g(H)$ is the LM hallucination rate, and $\widehat{MF}$ is defined in \cref{eq:mf}.
\end{corollary}

Now, we can state a weaker guarantee for semi-regular facts alone.
\begin{corollary}\label{cor:balfact}
Fix any $\delta\in [0,1], b, n \in \nats, r, \sparsity \in \reals$ and any $\sparsity$-sparse $D_\mathrm{world}$ with $r$-regular-facts. Then for \textit{any} algorithm $\A: X^n \rightarrow \Delta(X)$, with probability $\ge 1-\delta$ over $D_L\sim D_\mathrm{world}$ and $\trainvec \sim D_L^{\times n}$,
$$g(H) \ge \widehat{MF} -\mc{b}{p}{g} -\frac{3r ne^{-\sparsity}}{\delta}-\sqrt{\frac{6\ln(6/\delta)}{n}}.$$
\end{corollary}
The above is meaningful when sparsity $s \gg \log n$ is larger than the log of the number of training data. Otherwise, following bound uses semi-regularity of facts and probabilities.
\begin{corollary}\label{cor:general}
Fix any $\delta\in [0,1], b,n \in \nats, r, \sparsity \in \reals$ and any $\sparsity$-sparse $D_\mathrm{world}$ with $r$-regular-facts and $r$-regular-probabilities. Then for \textit{any} algorithm $\A: X^n \rightarrow \Delta(X)$, with probability $\ge 1-\delta$ over $D_L\sim D_\mathrm{world}$ and $\trainvec \sim D_L^{\times n}$,
$$g(H) \ge \widehat{MF}-\mc{b}{p}{g} -\frac{3r e^{-\sparsity}}{\delta}-\sqrt{\frac{6\ln(6/\delta)}{n}}.$$
\end{corollary}
It is easy to see that \cref{cor:1} is a special case of this corollary for $r=1$.

\subsection{Different types of facts} 
Our analysis immediately generalizes to multiple distinct types of facts  (e.g., article references and social media posts). Suppose there are $k > 1$ sets of factoids, $Y_1, Y_2, \ldots, Y_k$ and functions $f_i: X \rightarrow Y_i$, with corresponding sets of facts and hallucinations $F_i \cup H_i = Y_i$, monofact estimates $\widehat{MF}_i \in [0,1]$ and miscalibration rates $\mc{i,b}{p}{g}$. One also would generalizes the notion of $\sparsity$-sparse to include the fact that, for each type of fact, $|F_i| \le e^{-\sparsity}|H_i|$ with probability 1 and similarly generalize regular facts to hold for each type of fact. Then \cref{cor:1} implies:
\begin{corollary}\label{cor:types}
Fix any $\delta\in [0,1], b, k, n\in \nats, \sparsity \in \reals$ and any $\sparsity$-sparse regular $D_\mathrm{world}$. Then for \textit{any} algorithm $\A: X^n \rightarrow \Delta(X)$, with probability $\ge 1-\delta$ over $D_L\sim D_\mathrm{world}$ and $\trainvec \sim D_L^{\times n}$,
$$g(H_i) \ge \widehat{MF}_i -\mc{i, b}{p}{g} -\frac{3ke^{-\sparsity}}{\delta}-\sqrt{\frac{6\ln(6k/\delta)}{n}} \text{ for }i=1,2,\ldots,k.$$
\end{corollary}
The proof follows trivially from \cref{cor:1}
using the union bound and $\delta/k$ failure probability for each type of fact.

\subsection{Analysis approach}
While our model supposes $D_L \sim D_\text{world}$ followed later by $\trainvec \sim \D_L^{\times n}$, for analysis purposes we imagine first selecting $\trainvec \sim D_\text{train}$ and then selecting $p \sim \nu_{\trainvec}$ where $\nu=\nu_{\trainvec}$ is defined to be the posterior distribution on $p$ given $\trainvec$. These two procedures result in identical joint distributions on $p, \trainvec$, but the latter is easier to analyze.  In particular, we show:
\begin{theorem}\label{thm:main}
For any $\nu\in \Delta(\Delta(Y))$, any $O\subseteq F \subseteq Y$, any $g \in \Delta(Y)$, and any partition $\Pi\in \P(Y)$,
$$\E_{p \sim \nu}\left[\left(p(U) - \bigl\|p^{\Pi}-g\bigr\|_{\mathrm{TV}} -g(H) \right)_+ \right] \le \max_{y \in U} \Pr_{p \sim \nu}[y \in F]+|O| \max_{y \in U}\E_\nu[p(y)],$$
where $H := Y \setminus F$ and $U:=Y\setminus O$.
\end{theorem}
The corollaries stated earlier follow directly from this theorem together with the definition of $\mc{b}{p}{g}:=\TV{\coarsen{p}{\V_b(g)}}{g}$, Markov's inequality for a non-negative random variable to show that the quantity in the expectation is small with high probability, which we combine with existing bounds on the Good-Turing estimator from \cref{sec:good} that show that $|\widehat{MF}-p(U)|=\tilde{O}(\sqrt{1/n})$. The quantities on the right hand side correspond to the arbitrary facts and arbitrary probability notions. The above theorem is enough to show the result for various binning strategies, such as fixed-width, which depend arbitrarily on $g$ (but not on $p$).

\section{Proof of Theorem \ref{thm:main}}\label{sec:analysis}

This section proves \cref{thm:main}.
\begin{lemma}\label{lem:meat}
Let $S \subseteq Y$ and let $p^\Pi$ be any coarsening of $p$. Then,
$$\E_\nu\left[\left(p(S)-p^\Pi(S)\right)_+\right]  \le |Y \setminus S|\cdot \max_{y \in S} \E_\nu[p(y)].$$
\end{lemma}
\begin{proof}
Suppose $q=\coarsen{p}{\Pi}$ for some partition $\Pi\in\P(Y)$. Then,
\begin{align*}
p(S)-q(S) &=  \sum_{B \in \Pi} p(S \cap B) -q(S \cap B)\\
           &=  \sum_B p(S \cap B) - \frac{p(B)}{|B|}|S \cap B|\\
           &\le\sum_B p(S \cap B) - \frac{p(S \cap B)}{|B|}|S \cap B|\\
           &=  \sum_B p(S \cap B)\frac{|B| - |S \cap B|}{|B|}\\           
           &=  \sum_B p(S \cap B)\frac{|B\setminus S|}{|B|}\\           
           &\le\sum_B \frac{p(S \cap B)}{|S \cap B|}|B \setminus S|.
\end{align*}
Since $a\le b\Rightarrow (a)_+ \le (b)_+$ and this last quantity is non-negative,
\begin{align*}
\left(p(S)-q(S)\right)_+           &\le\sum_B \frac{p(S \cap B)}{|S \cap B|}|B \setminus S|\\
\E_\nu\left[\left(p(S)-q(S)\right)_+\right] &\le \sum_B \E_\nu\left[\frac{p(S \cap B)}{|S \cap B|}\right] \,|B \setminus S|\\
           &\le \sum_B |B \setminus S|\, \max_{y \in S} \E_\nu[p(y)] \\
           &=|Y \setminus S|\, \max_{y \in S} \E_\nu[p(y)].
\end{align*}
\end{proof}



We are no ready to prove \cref{thm:main}.
\begin{proof}[Proof of \cref{thm:main}]
Let $q=\coarsen{p}{\Pi}$. By the definition of TV, 
\begin{equation}\label{eq:foo}
g(U) \ge q(U) - \TV{q}{g}.
\end{equation}
Since $Y \setminus F = H \subseteq U$, we have $H=U \setminus (F \cap U)$ and,
\begin{align*}
    g(H) &= g(U) - g(F \cap U) \\
         &\ge q(U) - \TV{q}{g} - g(F \cap U) & \text{ by \cref{eq:foo}}\\
         &= p(U) - (p(U) - q(U)) - \TV{q}{g} - g(F \cap U)\\
         &= p(U) -\TV{q}{g} - \bigl(p(U) - q(U) + g(F \cap U)\bigr).
\end{align*}
Rearranging terms gives $p(U)- g(H)-\TV{q}{g} \le p(U) - q(U) + g(F \cap U)$. 
Applying $a \le b \Rightarrow (a)_+ \le (b)_+$, 
$$\left(p(U)- g(H)-\TV{q}{g} \right)_+ \le \left(p(U) - q(U) + g(F \cap U)\right)_+ \le \left(p(U) - q(U)\right)_+ + g(F \cap U).$$
Thus, it suffices to prove:
\begin{equation}\label{eq:ts}
\E_\nu\left[g(F \cap U) + \bigl(p(U)-q(U)\bigr)_+\right] \le \max_{y \in U} \Pr_\nu[y \in F] + |O| \max_{y \in U} \E_{\nu}[p(y)].
\end{equation}
To this end, linearity of expectation implies that,
\begin{equation}\label{eq:prg}
\E_\nu[g(F \cap U)] = \sum_{y \in U} g(y)\Pr_\nu[y \in F] \le \max_{y \in U} \Pr_\nu[y \in F].
\end{equation}

By \cref{lem:meat} (with $S=U$),
$$\E_\nu\left[\bigl(p(U)-q(U)\bigr)_+\right] \le |O| \max_{y \in U} \E_{\nu}[p(y)]
.$$
Combining this with \cref{eq:prg} gives \cref{eq:ts}, as required.
\end{proof}

\section{Upper bounds on hallucination rate}\label{sec:upper}

Could one prove a much better lower bound than $\widehat{MF}$? If not in general, are there better lower bounds under various assumptions on $D_\text{world}$? In this section, we argue that a significantly better lower-bound is not possible by showing that for any $D_\text{world}$, there is a algorithm that is calibrated and hallucinates at a rate near the missing facts rate, equivalent its Good-Turing estimate. The conceptually simple algorithm is neither efficient nor a good LM, but it suffices to show that a better lower-bound is not possible.

The algorithm memorizes the facts in the training data and generates random facts from the training data and random unseen factoids uniformly at random, but at different rates.
\begin{enumerate}
\item Inputs: $\trainvec \in X^n$.
\item Let $O, U=Y \setminus O$ be the sets of observed and unobserved factoids in the training data, respectively, as defined in \cref{eq:ou} and compute $\widehat{MF}$, the fraction of factoids that appear exactly once in the training data, as defined in \cref{eq:mf}. Let $g \in \Delta(Y)$ be defined as
$$g(y):=\begin{cases}
        \frac{\widehat{MF}}{|U|} & \text{if } y \in U,\\
        \frac{1-\widehat{MF}}{|O|} & \text{if } y \in O.\\                
    \end{cases}$$
\item For each factoid $y \in Y$, select a corresponding document $d(y) \in X$ such that $f(d(y))=y$. To be specific, one can take $d(y)$ to be the lexicographically first document in $\{x \in X \mid f(x) = y\}$.
\item Output $D_{LM} = d \circ g$, i.e., the distribution which samples $y \sim g$ and then outputs $d(y).$
\end{enumerate}
It is easy to see that, by design, $g = f \circ D_{LM}$.
\begin{lemma}
    For any $\delta,\lambda \in [0,1],D_\mathrm{world},n \ge 1$,
    $$\Pr_{\trainvec \sim \D_\mathrm{train}}\left[g(H) \le \widehat{MF} \mathrm{\ and\ }\mc{\infty}{p}{g} \le 3\sqrt{\frac{\ln(4/\delta)}{n}}\right] \ge 1-\delta,$$
    where $g = f \circ D_{LM}$ for the above algorithm, $g(H)$ is its hallucination rate, and $\mc{\infty}{p}{g}$ is defined in \cref{eq:mcinfty}.
\end{lemma}
In other words, with high probability one can have a well-calibrated LM that hallucinates with probability close to $\widehat{MF}$. 
\begin{proof}
    There can either be one or two bins $B_z^g$ based on whether or not $\frac{\widehat{MF}}{|U|} = \frac{1-\widehat{MF}}{|O|}$. If they are equal then $\mc{\infty}{p}{g}=0$. In any case, 
    $$\mc{\infty}{p}{g} \le \frac{1}{2}|p(U)-g(U)| + \frac{1}{2}|p(O)-g(O)| = |p(U) - g(U)| = |p(U) - \widehat{MF}|.$$
    Since $\widehat{MF}$ is a Good-Turing estimator, by \cref{cor:simp}, with probability $\ge 1-\delta$ the above quantity is at most $3\sqrt{\frac{\ln(4/\delta)}{n}}$. 
    At the same time, with certainty, $$g(H) = \frac{\widehat{MF}}{|U|}|H \cap U| \le \widehat{MF}.$$ 
\end{proof}

\section{Proofs of Corollaries}\label{sec:corproofs}

To prove \cref{cor:1,cor:balfact,cor:general}, we use Markov's inequality on the expectation of a non-negative random variable $W$, which states that $\Pr[W \ge \E[W]/\delta]\le \delta$. In our case, $W:=(p(U) - \TV{p^\Pi}{g} - g(H))_+$ and thus using $\delta\rightarrow (2/3)\delta$ and \cref{thm:main} imply that for any partition $\Pi$:
    $$\Pr_{p \sim \nu}\left[p(U) - \TV{p^\Pi}{g} - g(H) \ge \frac{3}{2\delta}\left(\max_{y \in U} \Pr_{p \sim \nu}[y \in F]+|O| \max_{y \in U}\E_\nu[p(y)]\right) \right] \le \frac{2\delta}{3}.
    $$
Rearranging terms,
    \begin{equation}\label{eq:markov}
    \Pr_{p \sim \nu}\left[g(H) \ge p(U) - \TV{p^\Pi}{g} - \frac{3}{2\delta}\left(\max_{y \in U} \Pr_{p \sim \nu}[y \in F]+|O| \max_{y \in U}\E_\nu[p(y)]\right) \right] \ge 1-\frac{2\delta}{3}.
    \end{equation}
Also, \cref{cor:simp} of \cref{sec:good} with $\delta\rightarrow \delta/3$ implies that for any $\delta \in (0,1]$ and any $D_L$,
    \begin{equation}\label{eq:aa}
\Pr_{\trainvec \sim D_L^{\times n}}\left[p(U) \ge \widehat{MF} - \sqrt{\frac{6\ln(6/\delta)}{n}}\right]\ge 1-\frac{\delta}{3}.
\end{equation}
It is now straightforward to prove the corollaries. 
\begin{proof}[Proof of \cref{cor:1}]
    For a regular $D_\text{world}$, because it has 1-regular-facts and 1-regular-probabilities, with probability 1 the posterior satisfies:
    $$\max_{y \in U} \Pr_{p \sim \nu}[y \in F]+|O| \max_{y \in U}\E_\nu[p(y)] \le \frac{\E[|F\cap U|]}{|U|} + \frac{|O|}{|U|}\E[p(U)] \le 2 \frac{|F|}{|U|} \le 2e^{-\sparsity}.$$
    In the above we have used the fact that $O \subseteq F$ and $U \supseteq H$.
    The proof follows immediately from this, \cref{eq:markov,eq:aa} and the union bound, using $\Pi=\V_b(g)$.
\end{proof}
The proofs of \cref{cor:balfact} and \cref{cor:general} also follow directly, where in \cref{cor:balfact} we use the fact that $|O|\le n$.

\section{Conclusions, limitations, and future work}\label{sec:conclusions}

When one first encounters LM hallucinations, it is perhaps surprising that a system which clearly embeds such a rich diverse array of detailed knowledge at the same time creates complete fabrications with no basis in fact or the training data. This work aims to demystify this phenomenon by showing that pretraining LMs for predictive accuracy leads to hallucination even in an ideal world where the training data is perfectly factual, there is no blur between facts and hallucinations, each document contains at most one fact, and there is not even a prompt that would encourage hallucination. Moreover, our theory explains why modern LMs hallucinate more than older LMs such as trigram models, even though both are trained on similar types of data with similar objectives.

The monofact rate may shed light on the rates at which calibrated LMs must hallucinate for different types of facts. One expects hallucination for facts that have a high monofact rate, i.e., the types of facts which often occur just once in the training data. Interestingly, this would not be common for references to books or articles, a problematic type of hallucination discussed today. Therefore, these may arise from other issues such as model capacity, when one considers the shear number of facts including references and others that an LM encounters in training. Furthermore, correcting for hallucinated references may be doable by modifying the pre-training pipeline without post-training, though this will not address other types of arbitrary facts where the monofacts are common, as in our 5W example.

There are several limitations to our work. First, we only study one statistical source of hallucination. There are many other types of hallucination and reasons LMs may hallucinate beyond pure statistics. Second, our semantic notion of calibration is different from the standard token-based ones used in classification. While natural and simple to define, our notion has the disadvantage of being computationally intractable to evaluate for many models. Third, factuality is not always clear-cut, facts are not all disjoint, and our regularity assumptions may not hold for facts that have a mild systematic component. As an example, if the training data contains the Alex Wilkins 5W fact from the introduction, then it is also follows that Alex Wilkins has eaten at Salumeria at some point, which is a different but overlapping fact. Finally, it could be the case that aspects of the real world, messy and different from our idealized setting, actually reduce the minimal hallucination rates. For instance, it could be that having multiple facts in a document makes models less likely to hallucinate and thus our lower bounds do not apply. 

In future work, it would be interesting to use the insights presented here to further reduce hallucination in LMs. An interesting question is how to convert a pretrained (calibrated) model to one that is good at factual prediction. A step in this process may be to distinguish systematic facts from arbitrary ones, which LMs may be capable of at some point in the future if not today. For example, for generation, one would not desire fabricated book titles, but one would like mathematical statements.  What is the difference between fabricating a non-existent book title from generating an inequality such as $17 < 124398$, if neither has ever been written down? Humans know that the latter can be manufactured (as long as it is mathematically correct) while the former cannot, which presumably is how we avoid hallucinating. It seems conceivable that LMs can similarly represent this distinction, and the work mentioned showing that LMs ``know'' when they are hallucinating suggests that this may indeed be possible with today's LMs.

\paragraph{Acknowledgments.} We thank Mirac Suzgun and Kevin Leyton-Brown for helpful discussions.

\bibliographystyle{ACM-Reference-Format}
\bibliography{adam_zotero}

\appendix

\section{Good-Turing estimator bounds}\label{sec:good}

The distribution bounds here are stated for a general set $S$, i.i.d.\ sample $s:=(s_1,s_2,\ldots, s_n)\in S^n$ from an arbitrary distribution $\D \in\Delta(S)$.

\begin{definition}[Missing mass]\label{def:mm}
For distribution $\D \in \Delta(S)$, $n \ge 1$, and sample $s \in S^n$, the missing mass is:
$$M_\D(s) := \D\left(S \setminus \{s_1,s_2\ldots, s_n\}\right).$$
\end{definition}
\begin{theorem}[\citep{berend_concentration_2013,mcallester_concentration_2003}]\label{thm:mm}
    For any distribution $\D \in \Delta(S)$ and any $n \ge 1$, $\eps \ge 0$, let $\overline{M}^n_\D:=\E_{s \sim \D^n}[M_\D(s)]$. Then:
    \begin{align}
    \Pr_{s \sim \D^n}[M_\D(s) \ge \overline{M}^n_\D + \eps] &\le e^{-\eps^2n}\label{eq:right}\\
    \Pr_{s \sim \D^n}[M_\D(s) \le \overline{M}^n_\D - \eps] &\le e^{-1.92\eps^2n}\label{eq:left}
    \end{align}
\end{theorem}
\cref{eq:right} is Theorem 16 of \citet{mcallester_concentration_2003} and Theorem 2.1 of \citet{berend_concentration_2013}. \cref{eq:left} is Theorem 2.2 of \citet{berend_concentration_2013}, though for simplicity we use the worse bound of $e^{-\eps^2n}$ which was also present in \citet{mcallester_concentration_2003}.

\begin{definition}[Good Turing estimator]\label{def:gt}
For $n \ge 1$, set $S$, and sample $s \in S^n$, the Good-Turing estimator \citep{good_population_1953} is,
$$GT(s) := \frac{1}{n}\bigl|\{i \in [n]\mid \forall j \ne i~s_i \ne s_j \}\bigr|.$$
\end{definition}

In words, the estimator above is defined as the fraction of elements of a sample each of which appears exactly once in the sample. 

\begin{lemma}[\citep{mcallester_concentration_2003}]\label{thm:gt}
    For any distribution $\D \in \Delta(S)$ and any $n \ge 1$, $\delta \in (0,1]$, let $\overline{GT}^n_\D:=\E_{s \sim \D^n}[GT_\D(s)]$. Then:
    \begin{align}
    \Pr_{s \sim \D^n}\left[GT_\D(s) \ge \overline{GT}^n_\D + \sqrt{\frac{2 \ln 1/\delta}{n}}\right] &\le \delta\label{eq:rightGT}\\
    \Pr_{s \sim \D^n}\left[GT_\D(s) \le \overline{GT}^n_\D - \sqrt{\frac{2 \ln 1/\delta}{n}}\right] &\le \delta\label{eq:leftGT}
    \end{align}
\end{lemma}
\cref{eq:rightGT} is Theorem 16 of \citet{mcallester_concentration_2003} and \cref{eq:leftGT} has the identical 1-line proof using McDiarmid's inequality.

The next lemma says that the expected values of the missing mass and unique elements in training data are very close. 

\begin{lemma}\label{lem:squash}
For any $n \ge 1$ and any $\D \in \Delta(S)$, 
$$\overline{M}^n_\D \le \overline{GT}^n_\D \le \overline{M}^n_\D + \frac{1}{n} $$
for  $\overline{M}^n_\D$ as defined in \cref{thm:mm} and $\overline{GT}^n_\D$ as defined in \cref{thm:gt}.
\end{lemma}

\begin{corollary}\label{cor:conc}
    For any set $S$, distribution $\D\in \Delta(S)$, any $n \ge 1, \delta \in (0,1]$,
    \begin{align}
        \Pr_{s \sim \D^n}\left[~\left|M_\D(s)-GT(s)\right| \le \frac{1}{n} + 2.42\sqrt{\frac{\ln(4/\delta)}{n}}~\right] &\ge 1-\delta.\label{eq:MMdouble}\\
        \Pr_{s \sim \D^n}\left[M_\D(s)\ge GT(s) - \frac{1}{n} - 2.14\sqrt{\frac{\ln(2/\delta)}{n}}~\right] &\ge 1-\delta.\label{eq:MMsingle}
    \end{align}
\end{corollary}
\begin{proof}
    \cref{eq:MMdouble} is established by setting $\eps=\sqrt{\ln(4/\delta)/n}$ in \cref{eq:right,eq:left}, which by the union bound implies,
    $$\Pr_{s \sim \D^n}\left[~\left|M_\D(s) - \overline{M}^n_\D \right| \ge \sqrt{\frac{\ln(4/\delta)}{n}}\right] \le \frac{\delta}{4} + \frac{\delta}{4} = \frac{\delta}{2}.$$
    Plugging $\delta'=\delta/4$ in \cref{eq:rightGT,eq:leftGT} and the union bound gives,
    $$\Pr_{s \sim \D^n}\left[~\left|GT(s) - \overline{GT}^n_\D \right| \ge \sqrt{\frac{2\ln(4/\delta)}{n}}\right] \le \frac{\delta}{4} + \frac{\delta}{4} = \frac{\delta}{2}.$$
    Combining the above two with \cref{lem:squash}, the triangle inequality, and the fact that $1+\sqrt{2} \le 2.41$ give \cref{eq:MMdouble}. 
    
    Similarly, \cref{eq:MMsingle} follows by using $\eps=\sqrt{\ln(2/\delta)/(1.92n)}$ in \cref{eq:left} and \cref{lem:squash} and \cref{eq:rightGT} with $\delta/2$ and summing the corresponding three inequalities to give,
    $$\Pr\left[M_\D(s) \ge GT(s) - \frac{1}{n} - \sqrt{\frac{\ln(2/\delta)}{1.92n}} - \sqrt{\frac{2\ln(2/\delta)}{n}}\right] \le \frac{\delta}{2}+\frac{\delta}{2}\le \delta.$$
    Using the fact that $\sqrt{1/1.92} + \sqrt{2}<2.14$ completes the proof.
\end{proof}

We now simplify the above expression.
\begin{corollary}\label{cor:simp}
    For any set $S$, distribution $\D\in \Delta(S)$, any $n \ge 1$,
    \begin{align}
        \forall \delta \in (0,1]~~\Pr_{s \sim \D^n}\left[~\left|M_\D(s)-GT(s)\right| \le 3\sqrt{\frac{\ln(4/\delta)}{n}}~\right] &\ge 1-\delta.\label{eq:simpdouble}\\
        \forall \delta \in (0,1/3]~~\Pr_{s \sim \D^n}\left[M_\D(s)\ge GT(s) - \sqrt{\frac{6\ln(2/\delta)}{n}}~\right] &\ge 1-\delta.\label{eq:simpsingle}
    \end{align}
\end{corollary}

\begin{proof}
    We first show \cref{eq:simpdouble}. Note that \cref{eq:simpdouble} holds trivially for $n \le 9$ because $GT(s), M_\D(s) \in [0,1]$ and $3 \sqrt{\ln(4)/9} > 1$. Thus, from Cor.~\ref{cor:conc}, it suffices to verify that for $n > 9, \delta \le 1$:
    \[
      \frac{1}{n} + 2.42\sqrt{\frac{\ln(4/\delta)}{n}} \le 3\sqrt{\frac{\ln(4/\delta)}{n}}
    \]
    In other words, we need
    \[
    0.58 \sqrt{\frac{\ln(4/\delta)}{n}} \ge \frac{1}{n}.
    \]
    Squaring and simplifying, this is
    \[
    n \ge \frac{1}{(0.58)^2 \ln(4/\delta)}.
    \]
    Since the RHS is a monotonic increasing function of $\delta$, we can use its largest value of $\delta=1$, and it suffices to have $n > 4.29$. 

    For \cref{eq:simpsingle}, note that it holds trivially for $n \le 6$ because $\sqrt{6 \ln(2/\delta)/n} \ge \sqrt{6 \ln(6)/6}>1$. Thus, from Cor.~\ref{cor:conc}, it suffices to verify that for $n > 6, \delta \le 1/3$:
    \[
      \frac{1}{n} + 2.14\sqrt{\frac{\ln(2/\delta)}{n}} \le \sqrt{\frac{6\ln(2/\delta)}{n}}
    \]
    In other words, we need
    \[
    (\sqrt{6}-2.14) \sqrt{\frac{\ln(2/\delta)}{n}} \ge \frac{1}{n}.
    \]
    Squaring and simplifying, this is
    \[
    n \ge \frac{1}{(\sqrt{6}-2.14)^2 \ln(2/\delta)}.
    \]
    Since the RHS is a monotonic increasing function of $\delta$, we can use its largest value of $\delta=1/3$, and it suffices to have $n > 5.83$. 
\end{proof}

\section{Generalizations and alternatives}\label{sec:alternatives}

There are alternative models one could consider. Here we discuss alternatives in terms of log-loss and prompts, as well as fixed-width (non-adaptive) binning strategies and other notions of calibration.

\subsection{Hallucination with Prompts}\label{sec:prompts}

Many uses of LMs require text to be generated conditionally based on a prefix string called a prompt. The degree of hallucination will heavily depend on the distribution of prompts and their lengths. Here we observe how prompts can be an additional source of hallucinations, but at the same time they can also make hallucinations unnecessary, depending on their distributions. The analysis we have performed earlier covers unconditional generation in which there are no prompts, which is like zero-length prompts.

First, observe that a certain distribution over prompts can make hallucination unnecessary. In particular, if the prompts themselves are very long, i.e., complete documents (ending with an end-of-document token if there is one), distributed exactly as $D_{LM}$, then any $D_{LM}$ which offers empty string completions will give statistically perfect completions and never hallucinate. More generally, this (arguably good) situation occurs whenever the facts are contained in the prompts and not their completions. 

Second, note that out-of-distribution (OOD) prompts can lead to significantly more hallucinations. For any LM, one can consider the single worst (adversarial) prompt $a$ which leads to the worst hallucination rate. (In the introduction we gave the example of the prompt \textit{The 15 Elves of San Salami are named}\ldots.) The prompt distribution could be concentrated on this single prompt in which case the hallucination rate would be terrible, assuming there is at least one adversarial prompt.

\subsection{KL-divergence and log-loss}
Clearly not all LMs hallucinate. Examples include an LM that avoids memorizing and regurgitating the training data or even an LM that only output the word \textit{yay} with probability 1. Such LMs would score poorly under predictive measures such as log-likelihood on held-out data, but it is not clear that LMs that perform well under log-likelihood must hallucinate. This suggests that one may perform the analysis in terms of log-loss $\E_{x \sim D_\text{train}}[-\log D_{LM}(x)]$, since LMs are generally trained to minimize this loss. Equivalently, one can consider the KL divergence which is a non-negative quantity that measures how far the log-loss is from its minimum achievable value, the entropy of $D_\text{train}$.
$\E_{x \sim D_\text{train}}[\log D_\text{train}(x) - \log D_{LM}(x)].$

One advantage of calibration analysis over KL-divergence is that, from a statistical perspective ignoring computation, one can generally expect to achieve calibration (miscalibration close to 0), e.g., by generating a uniformly random factoid. In contrast, one cannot expect to achieve KL divergence close to 0. Additionally, one can take any $D_{LM}$ which hallucinates and convert it to a model which does not without significantly increasing its KL divergence by, for example, mixing it with a model which regurgitates the training distribution. Specifically, consider a model $D_{LM}$ which, with probability 99\% outputs the word \textit{yay} and with probability 1\% outputs $x \sim D_{LM}$. This model hallucinates with probability $< 1\%$. Furthermore, the log-loss this new model at most $\log 100$ bits larger than that of $D_{LM}$. This difference is small difference on the scale of the entropy of documents, especially for longer documents since entropy typically scales linearly with the document length.

Nonetheless, it would be interesting to see if one can quantify hallucination rates purely on the basis of accuracy rather than calibration.
A starting point may be the work of \citet{feldman_does_2021} which shows that statistical models must memorize their training data for accuracy purposes.\footnote{We acknowledge Gautam Kamath for pointing out this connection.}

\subsection{Alternative calibration definitions}\label{sec:calvar}

In this section, we define a more standard alternative definition of calibration based on log-probability bins of equal width. Recall that $B_z := \{ y \in Y \mid g(y)=z\}$ and $B_I:= \{y \in Y \mid g(y) \in I\}.$ We now define bins of fixed width in probability space.
\begin{definition}[Binning]
For $\eps \in (0,1)$, the binning $\B(g, \eps)$ with equally spaced bins in log-probability space, is the following partition: 
\begin{equation}\label{eq:bineps}
\B(g, \eps) := \left\{ B_{\left((1-\eps)^{i+1}, (1-\eps)^i\right]} \,\middle|\, i=0,1,2, \ldots\right\} \cup \left\{B_0\right\}.
\end{equation}
For $\eps \in \{0,1\}$, let $\B(g, 0):=\B(g)=\{B_z\mid z\in [0,1]\}$ and $\B(g,1):=\bigl\{B_{[0,1]}\bigr\}=\{Y\}$.
\end{definition}
Thus $\eps$ determines the bid widths on a log-scale, with small $\eps$ corresponding to narrow bins. Thus one could use $\TV{\coarsen{p}{\B(g, \eps)}}{g}$ as a definition of miscalibration and the corresponding corollary would follow directly from our previous analysis.

\begin{corollary}\label{cor:1b}
Fix any $\delta\in [0,1], n\in \nats, \sparsity \in \reals, \eps\in(0,1)$ and any $\sparsity$-sparse regular $D_\mathrm{world}$. Then for \textit{any} algorithm $\A: X^n \rightarrow \Delta(X)$, with probability $\ge 1-\delta$ over $D_L\sim D_\mathrm{world}$ and $\trainvec \sim D_L^{\times n}$,
$$g(H) \ge \widehat{MF} -\TV{\coarsen{p}{\B(g, \eps)}}{g}  -\frac{3e^{-\sparsity}}{\delta}-\sqrt{\frac{6\ln(6/\delta)}{n}}.$$
\end{corollary}
\begin{proof}
    The proof of this Corollary follows exactly the same proofs as that of \cref{cor:1} except that we use $\coarsen{p}{\B(g, \eps)}$ in place of $\V_b(g)$. 
\end{proof}

We next use an even more standard definition which is not based on statistical distance. 
Recall that our first definition of calibration, without limits on bins as in \cref{eq:mcinfty}, can be written as,
$$\mc{\infty}{p}{g}:=\bigl\|\coarsen{p}{\B(g)}-g\bigr\|_{\text{TV}} = \frac{1}{2}\sum_{B \in \B(g)}\sum_{y \in B}\left|\frac{p(B)}{B}-g(y)\right|=\frac{1}{2}\sum_{B \in \B(g)}\left|p(B)-g(B)\right|.$$
This is the most obvious definition and the question is how to generalize it to bins. The above also suggests the following alternative generative definition of calibration error.
\begin{definition}[Generative calibration error]\label{def:miscalibration} 
For $\eps \in [0,1]$, and distributions $p, g \in \Delta(Y)$, the $\eps$-\emph{generative calibration error} is,
$$\mis\eps{p}{g}:=\frac{1}{2}\sum_{B \in \B(g, \eps)} \bigl|p(B)-g(B)\bigr|.$$
\end{definition}
This definition means that $\mis{0}{p}{g}=\mc{\infty}{p}{g}$. Note that these two definitions of calibration error are related by the following lemma. 
\begin{lemma}\label{lem:misc}
Let $\eps \in [0,1]$. Then,
$$\TV{\coarsen{p}{\B(g, \eps)}}{g} - \eps \le \mis{\eps}{p}{g} \le \TV{\coarsen{p}{\B(g, \eps)}}{g}.$$
\end{lemma}
Before we prove \cref{lem:misc}, we observe that \cref{cor:1b} implies the following.
\begin{corollary}\label{cor:1c}
Fix any $\delta\in [0,1], n\in \nats, \sparsity \in \reals, \eps\in(0,1)$ and any $\sparsity$-sparse regular $D_\mathrm{world}$. Then for \textit{any} algorithm $\A: X^n \rightarrow \Delta(X)$, with probability $\ge 1-\delta$ over $D_L\sim D_\mathrm{world}$ and $\trainvec \sim D_L^{\times n}$,
$$g(H) \ge \widehat{MF} -\mis{\eps}{p}{g} -\eps  -\frac{3e^{-\sparsity}}{\delta}-\sqrt{\frac{6\ln(6/\delta)}{n}}.$$
\end{corollary}

We now return to prove \cref{lem:misc}.
\begin{proof}[Proof of \cref{lem:misc}]
Let $\Pi = \B(g,\eps)$. Then,
\begin{align*}
    \mis\eps{p}{g} &= \frac{1}{2}\sum_{B \in \Pi} | p(B) - g(B)|\\
                &= \frac{1}{2}\sum_{B \in \Pi} \sum_{y\in B}  \frac{1}{|B|} \left| p(B) - g(B) \right|\\
                &= \frac{1}{2}\sum_{B \in \Pi} \sum_{y\in B}   \left| \frac{p(B)}{|B|} - \frac{g(B)}{|B|} \right|\\
                &= \frac{1}{2}\sum_{B \in \Pi} \sum_{y\in B}   \left| \coarsen{p}{\Pi}(y) - \coarsen{g}{\Pi}(y) \right|\\
                &= \frac{1}{2}\sum_{B \in \Pi} \sum_{y\in B} \left|\coarsen{p}{\Pi}(y)-g(y) +   g(y) - \coarsen{g}{\Pi}(y) \right|\\          
                &\ge \frac{1}{2}\sum_{B \in \Pi} \sum_{y\in B} \left|\coarsen{p}{\Pi}(y)-g(y)\right| - \left| \coarsen{g}{\Pi}(y) - g(y) \right| & \text{by $|a+b| \ge |b| - |a|$}\\
                &=\TV{\coarsen{p}{\Pi}}{g} - \TV{\coarsen{g}{\Pi}}{g}.
\end{align*}
This proves the RHS inequality of the lemma. Thus it suffices to show $\TV{\coarsen{g}{\Pi}}{g}\le \eps$. We first claim that for all $y $:
\begin{equation}\label{eq:showme}
\coarsen{g}{\Pi}(y)-g(y) \le \eps \coarsen{g}{\Pi}(y).
\end{equation}
Let $B\in \Pi$ be the bin containing $y\in B$. 
Now, recall that each bucket can be written as:
$$B^g_I=\bigl\{y ~\bigm|~ g(y) \in I\bigr\} \text{ for some interval }I \subseteq [0,1].$$
If $I=[0,0]$, then \cref{eq:showme} is trivially true because $g(y)=0=\coarsen{g}{\Pi}(y)$. Otherwise, say $I=((1-\eps)^(i+1), (1-\eps)^i]$ for some $i \ge 0$. Then, by definition of $\coarsen{g}{\Pi}$,
$$\coarsen{g}{\Pi}(y) = \frac{g(B)}{|B|} \in I,$$
because the weighted average of an numbers in an interval is also contained in the interval. Since this interval has (multiplicative) width $e^{-\eps}$, 
$g(y) \ge (1-\eps) \coarsen{g}{\Pi}(y).$ Equivalently, $$\coarsen{g}{\Pi}(y) - g(y) \le \eps \coarsen{g}{\Pi}(y).$$ 
Thus we have established \cref{eq:showme} which trivially implies that,
$$\forall y\in Y~\left(\coarsen{g}{\Pi}(y) - g(y)\right)_+ \le \eps \coarsen{g}{\Pi}(y).$$
Therefore, $$\TV{\coarsen{g}{\Pi}}{g} = \sum_{y\in Y}\left(\coarsen{g}{\Pi}(y) - g(y)\right)_+ \le  \eps \sum_{y\in Y}\coarsen{g}{\Pi}(y)=\eps,$$
which is all that remained to prove the lemma.\end{proof}

\end{document}